\documentclass[numbers,webpdf,imaiai]{ima-authoring-template}
\usepackage{amsmath,amsthm,amssymb,amsxtra,mathtools,bm}
\usepackage{array}
\usepackage{graphicx}
\usepackage{algorithm}
\usepackage{algpseudocode}
\usepackage{physics}
\usepackage{MnSymbol,wasysym}
\usepackage{array,tabularx,multirow}
\usepackage{makecell}
\usepackage{array}

\usepackage{array}

\usepackage{verbatim}
\usepackage{enumerate}
\usepackage{parskip}

\usepackage{color,soul}
\definecolor{clemson-orange}{RGB}{234,106,32}
\definecolor{highlight-orange}{RGB}{255,150,150}
\definecolor{chicago-maroon}{RGB}{128,0,0}
\definecolor{cincinnati-red}{RGB}{190,0,0}
\definecolor{soft-cyan}{RGB}{68,85,90}
\definecolor{firebrick}{RGB}{178,34,34}
\definecolor{crimson}{RGB}{220,20,60}
\definecolor{cerrulean}{rgb}{0.165,0.322,0.745}
\definecolor{jaam}{rgb}{0.45,0.0,0.45}

\usepackage{hyperref}
\hypersetup{
  colorlinks   = true,    
  urlcolor     = jaam,    
  linkcolor    = firebrick,    
  citecolor    = blue      
}

\usepackage{parskip}

\usepackage{thmtools}
\declaretheoremstyle[
    headfont=\bfseries, 
    bodyfont=\normalfont\itshape, spaceabove=10pt,
    spacebelow=10pt]{mystyle}
\theoremstyle{mystyle}
\newtheorem{theorem}{Theorem}[section]
\newtheorem{lemma}[theorem]{Lemma}
\newtheorem{corollary}[theorem]{Corollary}
\newtheorem{definition}{Definition}
\newtheorem*{remark}{Remark}

\newif\ifsolutions \solutionstrue

\def\final{0}
\newcommand{\reviewer}[3]{
  \expandafter\newcommand\csname #1\endcsname[1]{
    \ifthenelse{\equal{\final}{1}} {
      \textcolor{#3}{}
    } {
      \textcolor{#3}{\begin{center} \textbf{#2} ##1 \end{center}}
    }
  }
}

\reviewer{anirbit}{}{jaam}

\newcommand{\poly}{\textrm{poly}}

\renewcommand{\ip}[2]{\left\langle#1,#2\right\rangle}

\newcommand{\relu}{\mathop{\mathrm{ReLU}}}

\def\ve#1{\mathchoice{\mbox{\boldmath$\displaystyle\bf#1$}}
{\mbox{\boldmath$\textstyle\bf#1$}}
{\mbox{\boldmath$\scriptstyle\bf#1$}}
{\mbox{\boldmath$\scriptscriptstyle\bf#1$}}}

\newcommand{\x}{{\ve x}}
\newcommand{\y}{{\ve y}}
\newcommand{\z}{{\ve z}}
\renewcommand{\v}{{\ve v}}
\newcommand{\g}{{\ve g}}
\newcommand{\e}{{\ve e}}
\renewcommand{\u}{{\ve u}}
\renewcommand{\a}{{\ve a}}
\renewcommand{\c}{{\ve c}}
\newcommand{\f}{{\ve f}}

\newcommand{\m}{{\ve m}}
\newcommand{\p}{{\ve p}}
\renewcommand{\t}{{\ve t}}
\newcommand{\w}{{\ve w}}
\renewcommand{\b}{{\ve b}}
\renewcommand{\d}{{\ve d}}

\newcommand{\h}{{\ve h}}

\newcommand{\A}{\textrm{A}}

\newcommand{\W}{\textrm{W}}

\newcommand{\s}{\textrm{s}}

\def\1{\bm{1}}

\def\rvx{{\mathbf{x}}}

\def\va{{\bm{a}}}
\def\vb{{\bm{b}}}
\def\vc{{\bm{c}}}
\def\vd{{\bm{d}}}
\def\ve{{\bm{e}}}
\def\vf{{\bm{f}}}
\def\vg{{\bm{g}}}
\def\vh{{\bm{h}}}
\def\vi{{\bm{i}}}
\def\vj{{\bm{j}}}
\def\vk{{\bm{k}}}
\def\vl{{\bm{l}}}
\def\vm{{\bm{m}}}
\def\vn{{\bm{n}}}
\def\vo{{\bm{o}}}
\def\vp{{\bm{p}}}
\def\vq{{\bm{q}}}
\def\vr{{\bm{r}}}
\def\vs{{\bm{s}}}
\def\vt{{\bm{t}}}
\def\vu{{\bm{u}}}
\def\vv{{\bm{v}}}
\def\vw{{\bm{w}}}
\def\vx{{\bm{x}}}
\def\vy{{\bm{y}}}
\def\vz{{\bm{z}}}

\def\mW{{\bm{W}}}

\DeclareMathAlphabet{\mathsfit}{\encodingdefault}{\sfdefault}{m}{sl}
\SetMathAlphabet{\mathsfit}{bold}{\encodingdefault}{\sfdefault}{bx}{n}

\usepackage{amsmath,amsfonts,bm}

\def\eqref#1{equation~\ref{#1}}

\def\1{\bm{1}}

\def\rvx{{\mathbf{x}}}

\def\va{{\bm{a}}}
\def\vb{{\bm{b}}}
\def\vc{{\bm{c}}}
\def\vd{{\bm{d}}}
\def\ve{{\bm{e}}}
\def\vf{{\bm{f}}}
\def\vg{{\bm{g}}}
\def\vh{{\bm{h}}}
\def\vi{{\bm{i}}}
\def\vj{{\bm{j}}}
\def\vk{{\bm{k}}}
\def\vl{{\bm{l}}}
\def\vm{{\bm{m}}}
\def\vn{{\bm{n}}}
\def\vo{{\bm{o}}}
\def\vp{{\bm{p}}}
\def\vq{{\bm{q}}}
\def\vr{{\bm{r}}}
\def\vs{{\bm{s}}}
\def\vt{{\bm{t}}}
\def\vu{{\bm{u}}}
\def\vv{{\bm{v}}}
\def\vw{{\bm{w}}}
\def\vx{{\bm{x}}}
\def\vy{{\bm{y}}}
\def\vz{{\bm{z}}}

\def\mW{{\bm{W}}}

\DeclareMathAlphabet{\mathsfit}{\encodingdefault}{\sfdefault}{m}{sl}
\SetMathAlphabet{\mathsfit}{bold}{\encodingdefault}{\sfdefault}{bx}{n}

\newcommand{\E}{\mathbb{E}}

\newcommand{\R}{\mathbb{R}}

\newcommand{\reg}{\lambda}

\newcommand{\normltwo}{L^2}

\graphicspath{{plots}}

\begin{document}
\renewcommand{\a}{\va}
\renewcommand{\b}{\vb}
\renewcommand{\c}{\vc}
\renewcommand{\d}{\vd}
\renewcommand{\e}{\ve}
\renewcommand{\f}{\vf}
\renewcommand{\g}{\vg}
\renewcommand{\h}{\vh}
\renewcommand{\i}{\vi}
\renewcommand{\j}{\vj}
\renewcommand{\k}{\vk}
\renewcommand{\l}{\vl}
\renewcommand{\m}{\vm}
\newcommand{\n}{\vn}
\renewcommand{\o}{\vo}
\renewcommand{\p}{\vp}
\newcommand{\q}{\vq}
\renewcommand{\r}{\vr}
\renewcommand{\s}{\vs}
\renewcommand{\t}{\vt}
\renewcommand{\u}{\vu}
\renewcommand{\v}{\vv}
\renewcommand{\w}{\vw}
\renewcommand{\x}{\vx}
\renewcommand{\y}{\vy}
\renewcommand{\z}{\vz}
\DOI{10.1093/imaiai/iaae035}
\copyrightyear{20XX}
\vol{XX}
\pubyear{20XX}
\access{Advance Access Publication Date: Day Month Year}
\appnotes{Paper}
\copyrightstatement{Published by Oxford University Press on behalf of the Institute of Mathematics and its Applications. All rights reserved.}
\firstpage{1}

\title[Global Convergence of SGD On Two Layer Neural Nets]{Global Convergence of SGD On Two Layer Neural Nets}

\author{Pulkit Gopalani$^\dagger$
\address{\orgdiv{Computer Science \& Engineering}, \orgname{University of Michigan, Ann Arbor}, \orgaddress{\country{US}}}}
\author{Anirbit Mukherjee*
\address{\orgdiv{Computer Science}, \orgname{The University of Manchester}, \orgaddress{\street{Oxford Street}, \state{Manchester}, \country{U.K.}}}}

\corresp[*]{Corresponding Author: \href{mailto:anirbit.mukherjee@manchester.ac.uk}{anirbit.mukherjee@manchester.ac.uk}}

\received{3}{2}{2023}
\revised{11}{2}{2024}
\accepted{5}{12}{2024}

\editor{Associate Editor: Name}

\abstract{
In this note, we consider appropriately regularized $\ell_2-$empirical risk of depth $2$ nets with any number of gates and show bounds on how the empirical loss evolves for SGD iterates on it -- for arbitrary data and  if the activation is adequately smooth and bounded like sigmoid and tanh. This in turn leads to a proof of global convergence of SGD for a special class of initializations. We also prove an exponentially fast convergence rate for continuous time SGD that also applies to smooth unbounded activations like SoftPlus. Our key idea is to show the existence of Frobenius norm regularized loss functions on constant-sized neural nets which are ``Villani functions'' and thus be able to build on recent progress with analyzing SGD on such objectives. Most critically the amount of regularization required for our analysis is independent of the size of the net.{\let\thefootnote\relax\footnote{{An extended abstract based on this work has been accepted at the Conference on the Mathematical Theory of Deep Neural Networks (DeepMath) 2022}}}
{\let\thefootnote\relax\footnote{{Part of the work was done when author was at the department of Electrical Engineering, IIT Kanpur, India.}}}
}

\maketitle
\section{Introduction}

Modern developments in artificial intelligence have been significantly been driven by the rise of deep-learning - which in turn has been caused by the fortuitous coming together of three critical factors, (1) availability of large amounts of data (2) increasing access to computing power and (3) methodological progress. This work is about developing our understanding of some of the most ubiquitous methods of training nets. In particular, we shed light on how regularization can aid the analysis and help prove convergence to global minima for stochastic gradient methods for neural nets in hitherto unexplored and realistic parameter regimes.  

In the last few years, there has been a surge in the literature on provable training of various kinds of neural nets in certain regimes of their widths or depths, or for very specifically structured data, like noisily realizable labels. Motivated by the abundance of experimental studies it has often been surmised that Stochastic Gradient Descent (SGD) on neural net losses -- with proper initialization and learning rate -- converges to a low--complexity solution, one that generalizes --  when it exists \cite{cbmm}. But, to the best of our knowledge a convergence result for any stochastic training algorithm for even depth $2$ nets (one layer of activations with any kind of non--linearity), without either an assumption on the width or the data, has remained elusive so far.

In this work, we not only take a step towards addressing the above question in the theory of neural networks but we also do so while keeping to a standard algorithm, the Stochastic Gradient Descent (SGD). In light of the above, our key message can be summarily stated as follows,

\begin{theorem}[Informal Statement of Corollary \ref{thm:sgd-sig}]\label{thm:summarymain} If the initial weights are sampled from an appropriate class of distributions (dependent on the choice of accuracy parameter $\epsilon$), then for nets with a single layer of sigmoid or tanh gates -- for arbitrary data and size of the net --  SGD on $\ell_2-$losses on such architectures regularized with Frobenius norm of weights  with coefficient $\lambda > \lambda_c$ for a width independent constant $\lambda_c$, while using constant steps of size ${\mathcal O}(\epsilon)$, will converge in ${\mathcal O}(\frac{1}{\epsilon})$ steps to weights at which the expected regularized loss would be $\epsilon$--close to its global minimum.

\end{theorem}

The above follows from a more general result that we show in Theorem \ref{lem:error_bound}. At the crux of our analysis is the crucial observation informally stated in the following lemma - which infact holds for more general nets than what is encompassed by the above theorem,

\begin{lemma}
It is possible to add a constant amount of Frobenius norm regularization on the weights, to the standard $\ell_2-$loss on depth-$2$ nets with activations like SoftPlus, sigmoid and tanh gates s.t with no assumptions on the data or the size of the net, the regularized loss would be a Villani function. 
\end{lemma}

We note that the threshold amount of regularization needed in the above results is {\em not} dependent on the size/width of the net and can be shown to scale s.t it can be made arbitrarily small by choosing outer layer weights that are proportionately small. Since our convergence result stated in Theorem \ref{thm:summarymain} does not require any assumptions on the data or the neural net width, we posit that this significantly improves on previous work on proving convergence of SGD for shallow neural nets. 

The key idea that motivates the above lemma is that it was shown in \cite{weijie_sde} how the loss function satisfying the Villani condition is sufficient for a certain SDE to converge to its global minima - and this SDE was such that Proposition 3.5 therein showed that the loss along the SDE solution can be provably close in expectation to the loss at corresponding stages of the natural SGD. But the work in \cite{weijie_sde} did not give any explicit use-case in Machine Learning where this proof idea could be used -- and our key contribution is to unravel an opportune use case of this in a very standard setup with training neural nets.

To the best of our knowledge, similar convergence guarantees in the existing literature require either some minimum neural net width -- growing w.r.t. inverse accuracy and the training set size (NTK regime \cite{bach_lazy, du_provable}), infinite width (Mean Field regime \cite{chizat2018global,chizat2022, montanari_pnas}) or other assumptions on the data when the width is parametric (e.g. realizable data \cite{rongge_2nn_1, rongge_2nn_2}). In contrast to all these, we show that with appropriate $\ell_2$ regularization, SGD on mean squared error (MSE) loss on 2--layer sigmoid / tanh nets converges to the global infimum. Our critical observation towards this proof is that the above loss on 2--layer nets -- for a broad class of activation functions --- is a  ``Villani function''. Our proof get completed by leveraging the relevant results in \cite{weijie_sde}.

\paragraph{\textbf{Organization}} 

\quad In Section \ref{sec:rev} we will give a review of the various approaches towards provable learning of neural nets. In Section \ref{sec:mainthm} we present our primary result, Theorem \ref{lem:error_bound} which shows a guarantee on the loss value of SGD iterates and for gates like sigmoid and tanh - for arbitrary initialization. For special initializations a convergence results follows from it in Corollary \ref{thm:sgd-sig}. In Theorem \ref{thm:softplus} we also point out that if using the SoftPlus activation, we can show that the underlying SDE can converge in expectation to the global minimizer in linear time. In Section \ref{sec:weijierev}, we give a brief overview of the methods in  \cite{weijie_sde}, leading up to the proof of Theorem \ref{lem:error_bound} in Section \ref{sec:proof_mainthm}.  In Section \ref{sec:experiments} we discuss some experimental demonstrations that our regularizer does not overshadow the original loss function. We end in Section \ref{sec:conc} with a discussion of various open questions that our work motivates. The three appendices give various  calculations needed in the proof of the main theorem.

\section{Related Work}\label{sec:rev}

\paragraph{{\rm \bf Review of the NTK Approach To Provable Neural Training :}} One of the most popular parameter zones for theory has been the so--called ``NTK'' (Neural Tangent Kernel) regime -- where the width is a high degree polynomial in the training set size and inverse accuracy (a somewhat {\it unrealistic} regime) and the net's last layer weights are scaled inversely with width as the width goes to infinity,  \cite{lee2017deep,wu2019global,du2018gradient,su2019learning,kawaguchi2019gradient,huang2019dynamics,allen2019convergenceDNN,allen2019learning,allen2019convergenceRNN,du2018power,zou2018stochastic,zou2019improved,arora2019exact,arora2019harnessing,li2019enhanced,arora2019fine, bach_lazy,du_provable}. The core insight in this line of work can be summarized as follows: for large enough width, SGD {\it with certain initializations} converges to a function that fits the data perfectly, with minimum norm in the RKHS defined by the neural tangent kernel -- which gets specified entirely by the initialization (which is such that the initial output is of order one). A key feature of this regime is that the net's matrices do not travel outside a constant radius ball around the starting point -- a property that is often not true for realistic neural net training scenarios.

In particular, for the case of depth $2$ nets with similarly smooth gates as we focus on, in \cite{ali_subquadratic} global convergence of gradient descent was shown using number of gates scaling sub-quadratically in the number of data - which, to the best of our knowledge, is the smallest known width requirement for such a convergence in a regression setup.  On the other hand, for the special case of training depth $2$ nets with $\relu$ gates on cross-entropy loss for doing binary classification, in \cite{telgarsky_ji} it was shown that one needs to blow up the width only poly-logarithmically with target accuracy to get global convergence for SGD. In there it was pointed out as an important open question to determine whether one can get such reduction in width requirement for the regression setting too. The result we present here can be seen as an affirmative answer to this question posed in \cite{telgarsky_ji}. Similar guarantees were shown to hold for deep neural nets in \cite{chen2021overparameterization}.

\paragraph{{\rm \bf Review of the Mean-Field Approach To Provable Neural Net Training :}} In a separate direction of attempts towards provable training of neural nets, works like \cite{chizat2018global} showed that a Wasserstein gradient flow limit of the dynamics of discrete time algorithms on shallow nets, converges to a global optimizer -- if the convergence of the flow is assumed. We note that such an assumption is very non-trivial because the dynamics being analyzed in this setup is in infinite dimensions -- a space of probability measures on the parameters of the net. Similar kind of non--asymptotic convergence results in this so--called `mean--field regime' were also obtained in \cite{montanari_pnas,montanari_2,congfang_meanfield,chizat2018global,chizat2022,tzen_raginsky,jacot_ntk,pmnguyen_meanfield,sirignano1,sirignano_lln,sirignano_clt,entropic_fictitious}. In a recently obtained generalization of these insights to deep nets, \cite{congfang_meanfield} showed convergence of the mean--field dynamics for ResNets \cite{resnet}. The key idea in the mean--field regime is to replace the original problem of neural training which is a non-convex optimization problem in finite dimensions by a convex optimization problem in infinite dimensions -- that of probability measures over the space of weights. The mean--field analysis necessarily require the probability measures (whose dynamics is being studied) to be absolutely--continuous and thus de facto it only applies to nets in the limit of them being infinitely wide. 

Thus we note, that in contrast to either of the above two families of results, for nets with a single layer of activations -- while assuming the same non-linearities as what the mean--field results use and while not assuming anything about the width of the net or the training data -- we show in Corollary \ref{thm:sgd-sig}, that SGD provably finds the global minima of certain appropriately regularized $\ell_2$ loss on such nets. We note that some results in the NTK regime also hold with weight-decay regularization, \cite{generalized_ntk}  while in many cases the mean--field results also need it \cite{montanari_pnas, chizat2022, tzen_raginsky}.

In the next subsection we shall give a brief overview of some of the attempts that have been made to get provable deep-learning at parametric width - and we shall point out how our aforementioned result fills an important gap among those works.

\paragraph{{\rm \bf Need And Attempts To Go Beyond Large Width Limits of Nets}}\label{sec:beyondntk}  
The essential proximity of the NTK regime to kernel methods and it being less powerful than finite nets has been established from multiple points of view  \cite{allen2019can,wei2019regularization}. 

In \cite{weijie_elastic}, the authors had given a a very visibly poignant way to see that the NTK limit is not an accurate representation of a lot of the usual deep-learning scenarios. Their idea was to define a notion of ``local elasticity'' -- when doing a SGD update on the weights using a data say $\x$, it measures the fractional change in the value of the net at a point $\x'$ as compared to $\x$. It's easy to see that this is a constant function for linear regression - as is what happens at the NTK limit (Theorem 2.1 \cite{lee2019wide}). But it has been shown in \cite{anirbit_elastic} that this local-elasticity function indeed has non-trivial time-dynamics (particularly during the early stages of training) when a moderately large neural net is trained on a $\ell_2-$ loss.  

In \cite{belkin_non_ntk} it was pointed out that the near-constancy of the tangent kernel might not happen for even very wide nets if there is an activation at the output layer -- but still linear time gradient descent convergence can be shown. A set of attempts have also been made bridge the gap between real-world nets and the NTK paradigm by adding terms which are quadratic in weights to the linear predictor that NTK considers, \cite{belkin_quadratic}, \cite{yubai_quadratic}, \cite{Hanin2020Finite}

On the other hand, recently in \cite{constant_labels}, asymptotic convergence of  SGD has been proven for $\ell_2-$loss on arbitrary $\relu$ architectures -- but for constant labels. Similar progress has also happened recently for G.D. in \cite{sourav_gd} and \cite{belkin_dnn}.

Specific to depth-2 nets -- as we consider here -- there is a stream of literature where analytical methods have been honed to this setup to get good convergence results without width restrictions - while making other structural assumptions about the data or the net. One of the earliest breakthroughs in this direction was made in \cite{anima_tensor}  and for the restricted setting of realizable labels they could provably get arbitrarily close to the global minima. For non-realizable labels they could achieve the same while assuming a large width but in all cases they needed access to the score function of the data distribution which is a computationally hard quantity to know.  In a more recent development, \cite{aravindan_realizable} have improved over the above to include $\relu$ gates while being restricted to the setup of realizable data and the its marginal distribution being Gaussian. 
\renewcommand{\A}{\mathbf{A}} 
\renewcommand{\W}{\mathbf{W}}

One of the first proofs of gradient based algorithms to be doing neural training for arbitrarily wide depth$-2$ nets appeared in \cite{prateek_realizable}. In \cite{rongge_2nn_1} convergence was proven for training depth-$2$ $\relu$ nets for data being sampled from a symmetric distribution and the training labels being generated using a `ground truth' neural net of the same architecture as being trained. For similar distributional setups, some of the current authors had in \cite{our_own}  identified classes of  depth--$2$ $\relu$ nets where they could prove linear-time convergence of training -- and they also gave guarantees in the presence of a label poisoning attack.  The authors in \cite{rongge_2nn_2} consider a so-called ``Teacher--Student'' setup of training depth $2$ nets with absolute value activations. In this work, authors can get convergence in poly$(d, \frac{1}{\epsilon})$ time, in a very restricted setup of assuming Gaussian data, initial loss being small enough, and the teacher neurons being norm bounded and `well--separated' (in angle magnitude).  The authors in \cite{eth_relu} get width independent  convergence bounds for Gradient Descent (GD) with ReLU nets, however at the significant cost of having the restrictions of being only an asymptotic guarantee and assuming an affine target function and one--dimensional input data. While being restricted to the Gaussian data and the realizable setting for the labels, an intriguing result in \cite{klivans_chen_meka} showed that fully poly-time learning of arbitrary depth 2 ReLU nets is possible if one is in the ``black-box query model''.

In summary, to the best of our knowledge, it has remained an unresolved challenge to show convergence of SGD on any neural architecture with a constant number of gates while neither constraining the labels nor the marginal distributional of the data to a specific functional form. {\it In this work, for certain natural neural net losses, we are able to resolve this in Corollary \ref{thm:sgd-sig}. Thus we take a step towards bridging this important lacuna in the existing theory of stochastic optimization for neural nets in general.}

\paragraph{{\rm \bf Related Work on Provable Training of Neural Networks Using Regularization}} Using a regularizer is quite common in deep-learning practice and in recent times a number of works have appeared which have established some of these benefits rigorously.  In particular, in \cite{wei2019regularization} it was shown that for a specific classification task (noisy--XOR) definable in any dimension $d$, no NTK based 2 layer neural net can succeed in learning the distribution with low generalization error in $o(d^2)$ samples, while in $O(d)$ samples one can train the neural net using Frobenius/$\ell_2-$norm regularization. The authors in \cite{preetum_optimal} showed that for a specific optimal value of the $\ell_2$- regularizer the double descent phenomenon can be avoided for linear nets - and that similar tuning is possible even for real world nets.

In the seminal work \cite{rrt}, it was pointed out that one can add a regularization of the above kind to a Lipschitz loss and make it satisfy the dissipativity condition so that Stochastic Gradient Langevin Dynamics (SGLD) provably converges to its global minima. But SGLD is seldom used in practice, and to the best of our knowledge it remains unclear if the observation in \cite{rrt} can be used to infer the same about SGD. Also it remains open if there exists neural net losses on which the above can be invoked. We note that the convergence time in \cite{rrt} for SGLD is $\mathcal{O}\left(\frac{1}{\epsilon^5}\right)$ using an $\mathcal{O} \left ( \epsilon^4 \right )$ learning rate, while in our Corollary \ref{thm:sgd-sig} SGD converges in expectation to the global infimum of the regularized neural loss in time, $\mathcal{O}\left(\frac{1}{\epsilon}\right)$ using a $\mathcal{O} \left ( \epsilon \right )$ step-length.

\section{Setup and Main Results}\label{sec:mainthm}  

We start with defining the neural net architecture, the loss function and the algorithm for which we prove our convergence result. 

\begin{definition}[{\bf Constant Step-Size SGD On Depth-2 Nets}]\label{def:sgd} 
Let, $\sigma : \R \rightarrow \R$ (applied elementwise for vector valued inputs) be atleast once differentiable activation function. Corresponding to it, consider the width $p$, depth $2$ neural nets with fixed outer layer weights $\a \in \R^p$ and trainable weights $\mathbf{W} \in \R^{p \times d}$ as, 
\[ \R^d \ni \x \mapsto f(\x;\, \a, \mathbf{W}) = \a^\top\sigma(\mathbf{W}\x) \in \R \]
Then corresponding to a given set of $n$ training data $(\x_i,y_i) \in \R^d \times \R$, with $\norm{\x_i}_2 \leq B_x, \abs{y_i} \leq B_y, ~i=1,\ldots,n$ define the individual data losses $\tilde{L}_i(\mathbf{W}) \coloneqq \frac{1}{2} \left(y_i - f(\x_i, \a;\mathbf{W}) \right)^2$. Then for any $\lambda >0$ let the regularized empirical risk risk be,
\[ \tilde{L}(\mathbf{W}) \coloneqq \frac{1}{n}\sum_{i=1}^n \tilde{L}_i(\mathbf{W}) + \frac{\lambda}{2} \norm{\mathbf{W}}_F^2 \]

We implement on above the SGD with step-size $s>0$ as, 
\[ \mathbf{W}^{k+1} = \mathbf{W}^k - \frac{s}{b}\sum_{i \in \mathcal{B}_k} \nabla \tilde{L}_i(\mathbf{W}^k) -  s\lambda \mathbf{W}^k \]where $\mathcal{B}_k$ is a randomly sampled mini-batch of size $b$. 
\end{definition}

\begin{definition}[\textbf{Properties of the Activation $\sigma$}]\label{def:sigma}
Let the $\sigma$ used in Definition \ref{def:sgd} be bounded s.t. $\abs{\sigma(x)} \leq B_{\sigma}$, $C^{\infty}$, $L-$Lipschitz and $L_{\sigma}'-$smooth. Further assume that $\exists$ a constant vector $\c$ and positive constants $B_\sigma, M_D$ and $M_D'$ s.t $\sigma(\mathbf{0}) = \c$ and $\forall x \in \R, \abs{\sigma'(x)} \leq M_D, \abs{\sigma''(x)} \leq M_D'$ .
\end{definition}

In terms of the above constants we can now quantify the smoothness of the empirical loss as follows,

\begin{lemma}\label{def:lambdavillani} 
In the setup of definition \ref{def:sgd} and \ref{def:sigma}, $\exists ~{\rm gLip}(\tilde{L}) > 0$ s.t. the empirical loss, $\tilde{L}$ is ${\rm gLip}(\tilde{L})-$smooth. Further, there exists a constant $\lambda_c := 2\, M_D L B_x^2 \norm{\a}_2^2$ s.t $\forall ~\lambda > \lambda_{c} ~\& ~s>0$, the Gibbs' measure $\sim \exp\left(-\frac{2 \tilde{L}}{s}\right)$ satisfies a Poincar\'e-type inequality with the corresponding constant $\lambda_s$ as given therein. Moreover, we can bound the smoothness coefficient  of the empirical loss as, \[{\rm gLip}(\tilde{L}) \leq \sqrt{p}\left({\norm{\a}_2 B_x} B_y L_{\sigma}' +  \sqrt{p}\norm{\a}_2^2 M_D^2 B_x^2 + {{p} \norm{\a}_2^2 B_x^2} M_D' B_{\sigma} + \lambda\right)\]
\end{lemma}

The precise form of the Poincar\'e-type inequality used above is detailed in Theorem \ref{def:lambdas}.

\begin{theorem}[\bf{Global Convergence of SGD on Sigmoid and Tanh Neural Nets of $2$ Layers for Any Width and Data, Arbitrary Initialization.}]\label{lem:error_bound}
We continue in the setup of Definitions \ref{def:sgd} and \ref{def:sigma} and Lemma \ref{def:lambdavillani}. For any $s > 0$, define the probability measure $\mu_{s} \coloneqq \frac{1}{Z_{s}} \exp\left({-\frac{2\tilde{L}(\mW)}{s}}\right)$, $Z_{s}$ being the normalization factor. Then, $\forall ~T > 0$ and desired accuracy $\epsilon > 0$, $\exists$ constants $A(\tilde{L})$, $B(T,\tilde{L})$ and $C(s,\tilde{L})$ s.t if the above SGD  is executed at a constant step-size 
\[s = s^*(\epsilon,T) \coloneqq \min\left(\frac{1}{{\rm gLip}(\tilde{L})}, \frac{\epsilon}{(A(\tilde{L}) + B(T,\tilde{L}))}\right)\] 
with the weights $\mW^0$ initialized from any distribution with p.d.f $\rho_{\rm initial} \in \normltwo(\frac{1}{\mu_{s^*}})$ and then, the error at the end of having taken  $k = \frac{T}{s^*}$ SGD steps can bounded as, 
\[
    {\E}\tilde{L}(\mW^k) - \min_{\mW}\tilde{L}(\mW) \leq \epsilon + C(s^*,\tilde{L}) \norm{\rho_{\rm initial} - \mu_{s^*}}_{\mu_{s^*}^{-1}} e^{ - s^*\reg_{s^*} \cdot k} 
\]

\end{theorem}

Note that $\norm{\cdot}_{\mu_s^{-1}}$ denotes the following (Sec. 5, \cite{weijie_sde}), \[\norm{g}_{\mu_s^{-1}} = \left(\displaystyle\int_{\mathbb{R}^d} \abs{g(\x)}^2 d\mu_{s^{-1}}\right)^{1/2} = \left(Z_s\displaystyle\int_{\mathbb{R}^d} \abs{g(\x)}^2 \exp\left(\frac{2 \tilde{L}(\x)}{s}\right) d\x\right)^{1/2}\] where we have used that $\mu_s^{-1} = Z_s\exp\left(\frac{2 \tilde{L}(\x)}{s}\right)$ ; $Z_s$ is the normalization factor. Further, $g \in L^2(\mu_s^{-1})$ iff $\norm{g}_{\mu_s^{-1}} < \infty.$

Proof for the above theorem is given in Sec. \ref{sec:proof_mainthm}. We note that in the above setup by sampling initial weights in a special way can also have a convergence guarantee as follows,

\begin{corollary}[{\bf Global Convergence of SGD on Sigmoid and Tanh Neural Nets of $2$ Layers for Any Width and Data}]\label{thm:sgd-sig}

 We continue in the setup of Definitions \ref{def:sgd}, \ref{def:sigma} and Lemma \ref{def:lambdavillani}. For any $s > 0$, define the probability measure $\mu_{s} \coloneqq \frac{1}{Z_{s}} \exp\left({-\frac{2\tilde{L}(\mathbf{W})}{s}}\right)$, $Z_{s}$ being the normalization factor. Then, $\forall ~T > 0,$ and desired accuracy, $\epsilon > 0$, $\exists$ constants $A(\tilde{L})$, $B(T,\tilde{L})$ and $C(s,\tilde{L})$ s.t if the above SGD  is executed at a constant step-size $s = s^* \coloneqq \min\left(\frac{1}{{\rm gLip}(\tilde{L})}, \frac{\epsilon}{2 \cdot (A(\tilde{L}) + B(T,\tilde{L}))}\right)$ with the weights $\mathbf{W}^0$ initialized from a distribution with p.d.f $\rho_{\rm initial} \in L^2(\frac{1}{\mu_{s^*}})$ and $\norm{\rho_{\rm initial} - \mu_{s^*}}_{\mu_{s^*}^{-1}}  \leq \frac{\epsilon}{2 \cdot C(s^*,\tilde{L}) } \cdot e^{ \lambda_{s^*} \cdot T}$ -- then, in expectation, the regularized empirical risk of the net, $\tilde{L}$ would converge to its global infimum, with the rate of convergence given as, 
\[ {\E}\tilde{L}(\mathbf{W}^\frac{T}{s^*}) - \inf_{\mathbf{W}}\tilde{L}(\mathbf{W}) \leq \epsilon.\]
\end{corollary}

\textit{Proof.} Follows by using the stepsize and initialization distance assumption in Theorem \ref{lem:error_bound}.

We make a few quick remarks about the nature of the above guarantee,

{\it Firstly,} we note that the ``time horizon'' $T$ above is a free parameter - which in turn parameterizes the choice of the step-size and the initial weight distribution. Choosing a larger $T$ makes the constraints on the initial weight distribution weaker at the cost of making the step-size smaller and the required number of SGD steps larger. But for any value of $T$, the above theorem guarantees that SGD, initialized from weights sampled from a certain class of distributions, converges in expectation to the global minima of the regularized empirical loss for our nets for any data and width, in time ${\mathcal O}(\frac{1}{\epsilon})$ using a learning rate of ${\mathcal O}(\epsilon)$.

{\it Secondly,} we note that the phenomenon of a lower bound on the regularization parameter being needed for certain nice learning theoretic property to emerge has been seen in kernel settings too, \cite{aos_regularizer}. 

Also, to put into context the emergence of a critical value of the regularizer for nets as in the above theorem, we recall the standard result, that there exists an optimal value of the $\ell_2-$regularizer at which the excess risk of the similarly penalized linear regression becomes dimension free (Proposition 3.8, \cite{bach_ltfp}). However, we recall that the quantities required for computing this ``optimal'' regularizer are not knowable while training and hence it is not practically implementable. Thus, we see that for linear regression one can define a notion of an ``optimal" regularizer and it remains open to investigate if such a similar threshold of regularization also exists for nets. Our above theorem can be seen as a step in that direction.

{\it Thirdly,} we note that the lowerbounds on training time of neural nets proven in works like  \cite{klivans_superpoly}  do not apply here since these are proven for SQ algorithms and SGD is not of this type.

\textit{Finally}, note that the threshold value of regularization computed above, $\lambda_c$, does not explicitly depend on the training data or the neural architecture, consistent with observations in \cite{bartlett_book, chaos}. It depends on the activation and scales with the norm of the input data and the outer layer of weights. For intuition, suppose we scale the outer layer weights s.t we always have $\norm{\a}_2 \cdot B_x =1$. Then this leads to $\lambda_c = 2 \cdot  M_D L.$ For the sigmoid activation, $\sigma_{\beta}(x) = \frac{1}{1+e^{-\beta x}}.$ we have, $M_D = L = \frac{\beta}{4}$ and hence the $\lambda_c$ in this case (say $\lambda_{c, \beta}^{si}$) is $ = \frac{\beta^2}{8}.$ Since $\beta = 1$ is the most widely used setting for the above sigmoid activation, we use the same for our experimental analysis. This results in, 
\begin{equation}\label{lambsigvil} 
    \lambda_{c, 1}^{si} = 0.125
\end{equation} In experiments (Section \ref{sec:experiments}) we demonstrate that the degradation in performance due to the above regularization is not very significant.

\subsection{Global Convergence of Continuous Time SGD on Nets with SoftPlus Gates}

In \cite{weijie_sde} it was pointed out that in the relatively `easier' case if we only want too get a non-asymptotic convergence rate for the continuous time dynamics, the smoothness of the loss function is not needed and only the Villani condition suffices. In this short section we shall exploit this to show convergence of continuous time SGD on $\tilde{L}$ with the activation function being the unbounded `SoftPlus'. Also, in contrast to the guarantee about SGD in the previous subsection here we shall see that the SDE converges exponentially faster i.e \textit{at a  linear rate}. 

\begin{definition}[SoftPlus activation]
\quad For $\beta > 0,$ $x\in \R,$ define the SoftPlus activation function as \[{\rm SoftPlus}_{\beta}(x) = \frac{1}{\beta}\log_e{\left(1 + \exp(\beta x)\right)}\]
\end{definition}
\begin{remark}
Note that $\lim_{\beta \rightarrow \infty}{\rm SoftPlus}_{\beta}(x) = {\rm ReLU}(x).$ Also note that for $f(x) = $ SoftPlus$_\beta(x)$, $f'(x) = \sigma_{\beta}(x)$ (sigmoid function as defined above) and hence $\abs{f'(x)} \leq M_D$ for $M_D = 1$ and $f(x)$ is $L-$Lipschitz for $L = 1$. 
\end{remark}
\renewcommand{\W}{\mathbf{W}}

\begin{theorem}[{\bf Convergence To Global Minima of Continuous Time SGD on Depth$-2$ { SoftPlus} Nets}]\label{thm:softplus} \quad We consider SGD with step-size $s$ on a Frobenius norm regularized $\ell_2$-empirical loss on depth$-2$ neural nets as specified in Definition \ref{def:sgd}, while using $\sigma(x) = {\rm SoftPlus}_{\beta}(x)$ for $\beta > 0$. Then for $\mu_s, \rho$ and $\lambda_s$ as in Theorem \ref{thm:sgd-sig} and $s \in (0, S)$ for any $S>0,$ $\exists ~G(S, \tilde{L})$ that quantifies the excess risk at the stationary point of the SDE as, \[\tilde{L}(\W(\infty)) - \min_{\W}\tilde{L} \leq G(S, \tilde{L})\,s\] and $\exists ~C(s, \tilde{L})$, an increasing function of $s$, that satisfies
\[\abs{\E \tilde{L}(\W_s(t)) - \E \tilde{L}(\W_s(\infty))} \leq C(s, \tilde{L}) \norm{\rho - \mu_s}_{\mu_s^{-1}}\, e^{-\lambda_s t}.\]

Further, for any step size $0 < s \leq \min\left\{\frac{1}{2G(S, \tilde{L})}, S\right\}$,  for $\lambda > \lambda_{c} \coloneqq 2\,M_DLB_x^2\norm{\a}_2^2$ ($M_D$ and $L$ being defined as in the remark above) and for $t \geq \frac{1}{\lambda_s} \log{\left(\frac{2\, C(s, \tilde{L}) \norm{\rho - \mu_s}_{\mu_s^{-1}}}{\epsilon}\right)}$ we have that,
\[\E\, \tilde{L}(\W(t)) - \min_{\W}\tilde{L}(\W) \leq \epsilon.\]
\end{theorem}

\begin{proof}
\quad The SoftPlus function is Lipschitz and using the same analysis as in Appendix \ref{villani_condition}, we can claim that for $\lambda > \lambda_{c}$ the loss function in section \ref{def:sgd} with Softplus activations is a Villani function (and hence confining, by definition). The result can then be read off using Corollary 3.3 in \cite{weijie_sde}.\end{proof}

\subsection{\textbf{An Experimental Study of the Effect of Regularization at Various Widths}}

For further understanding of the scope of the known theory, in here we present some experimental studies on depth $2$ nets with sigmoid gates and using the normalizations that correspond to the theoretically needed threshold value of the regularizer being $0.125$ (equation \ref{lambsigvil}). We simulate SGD based training of multiple neural nets (details about the data and other hyperparameters are given below the respective plots), across a range of $\lambda$ and neural net widths $p.$ As evident in the plot (Figs. \ref{fig:multiwidth_a}, \ref{fig:multiwidth_r}), the test loss values in our regularization regime ($\lambda > 0.125$) are comparatively only mildly deteriorating, compared to regularizers upto four orders of magnitude lower. To the best of our knowledge, a similar provable convergence guarantee is not known for $\lambda < 0.125,$ and hence the slightly larger test loss that we incur is only a minor trade-off.

Code for these experiments can be viewed at this \href{https://github.com/pulkitgopalani/villani_sgd.git}{Github repository.}

\begin{figure}[!htb]
    \centering
    \includegraphics[width=0.7\textwidth]{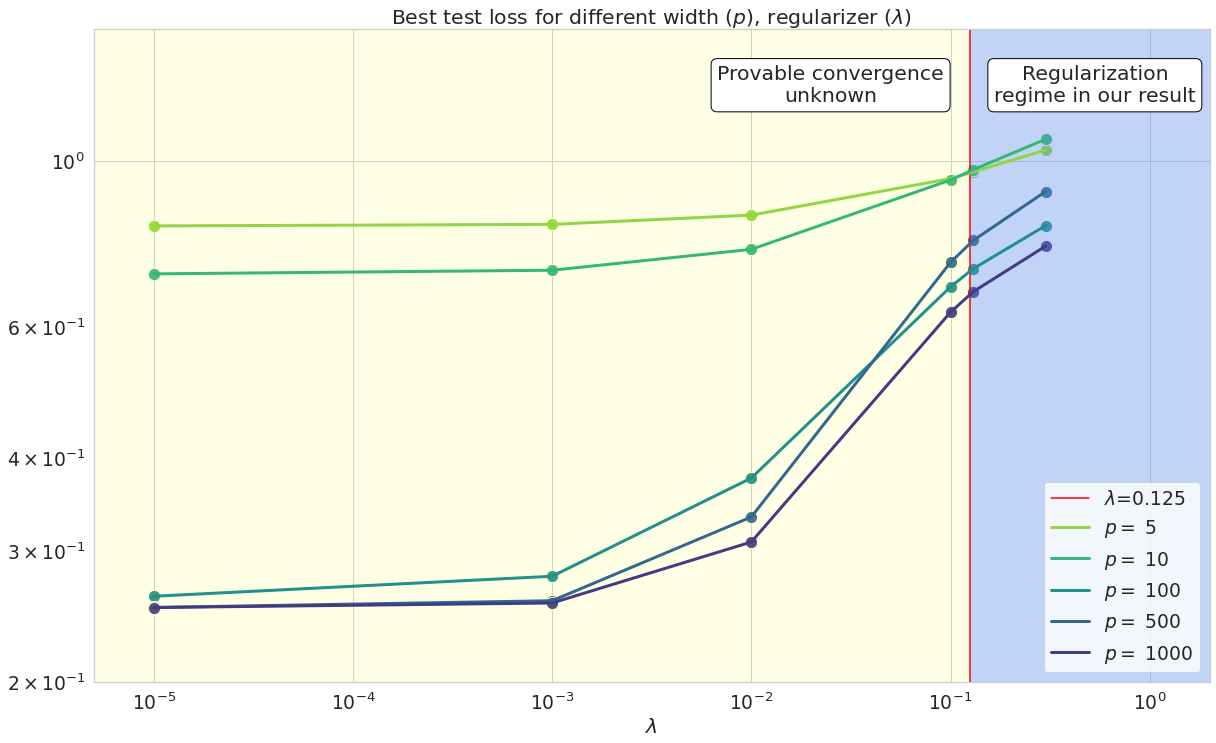}
    \caption{Best test loss across a range of $(\lambda,p = {\rm width})$ for $(\rvx,y)$ data labeled as,}
    \label{fig:multiwidth_a}
    \begin{center}
     $\rvx \sim {\rm Uniform} \left ( [0,1)^{20} \right ), y=\sin\left(\pi \frac{\norm{\rvx}_2^2}{20}\right) + \epsilon; \, \epsilon \sim \mathcal{N}(0, 0.25)$
\end{center}
\end{figure}

\begin{figure}[!htb]
    \centering
    \includegraphics[width=0.7\textwidth]{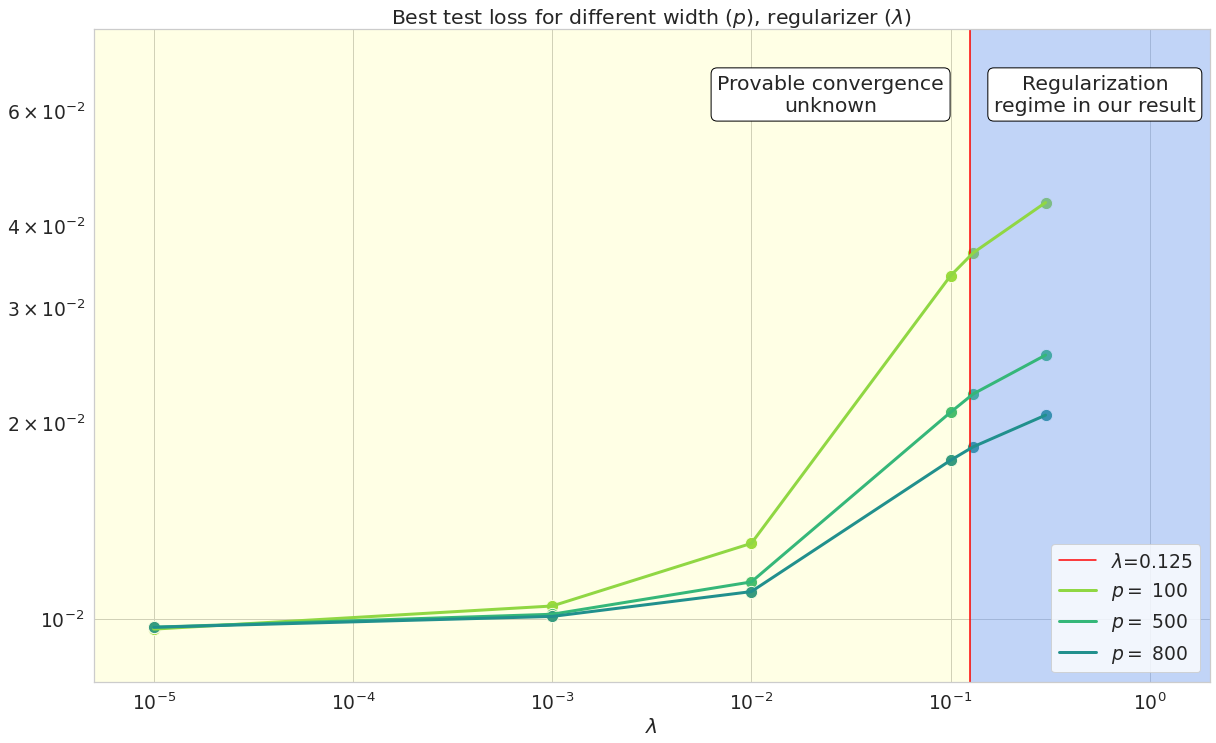}
    \caption{Best test loss across a range of $(\lambda,{\rm width} = p)$ for realizable data $(\rvx,y)$ sampled as,}
    \label{fig:multiwidth_r}
    \begin{center}
     $\rvx \sim {\rm Uniform} \left ( [0,1)^{20} \right ), y=\mathbf{a}^\top \sigma(\mathbf{W}\mathbf{\rvx}) + \epsilon;\, p=5, \epsilon \sim \mathcal{N}(0, 0.01), a_i \sim \frac{1}{\sqrt{p}} \mathcal{N}(0,1)$
\end{center}
\end{figure}

\section{Overview of \cite{weijie_sde}}\label{sec:weijierev}
\renewcommand{\W}{\mathbf{W}}
\newcommand{\bB}{\mathbf{B}}
 
In Section \ref{sec:proof_mainthm}, we give the proof for our main result (Theorem \ref{lem:error_bound}). As relevant background for the proof, we shall give in here a brief overview of the framework in  \cite{weijie_sde}, which can be summarized as follows : suppose one wants to minimize the function $\tilde{L}(\W) \coloneqq \frac{1}{n}\sum_{i=1}^n \tilde{L}_i(\W)$, where $i$ indexes the training data, $\W$ is the parameter space (the optimization space) of the loss function and $\tilde{L}_i$ is the loss evaluated on the $i^{th}-$datapoint. On this objective, a constant step-size mini-batch implementation of the Stochastic Gradient Descent (SGD) consists of doing the following iterates, $\W_{k+1} = \W_k - \frac{s}{b}\sum_{i} \nabla \tilde{L}_i(\W_k)$, where the sum is over a mini-batch (a randomly sampled subset of the training data) of size $b$ and $s$ is the fixed step-length. In, \cite{weijie_sde} the authors established that over any fixed time horizon $T >0$, as $s \rightarrow 0$, the dynamics of this SGD is arbitrarily well approximated (in expectation) by the Stochastic Differential Equation (SDE),
\begin{align}\label{def:sde}
\dd{\W_s(t)} = -\nabla \tilde{L}(\W_s(t)) \dd{t} + \sqrt{s}\dd{\bB(t)} \quad \quad \text{(SGD--SDE)}
\end{align} where $\bB(t)$ is the standard Brownian motion. We recall that the Markov semigroup operator $P_t$ for a stochastic process $X_t$ and its  infinitesimal generator $\mathcal{L}$  are given as, $P_t f (x) \coloneqq \E[f(X_t) \mid X_0 = x]$ and $\mathcal{L}f  \coloneqq \lim_{t\,  \downarrow\, 0}\,  \frac{P_t f - f}{t}$.

Thus for the SDE in Eq. \ref{def:sde}, $\mathcal{L}$ and its adjoint $\mathcal{L}^*$ are given as
\begin{align*}
    \mathcal{L}f = \frac{s}{2}\Delta f - \ip{\nabla f}{\nabla \tilde{L}} \implies
    \mathcal{L^*}f = \ip{\nabla {f}}{\nabla \tilde{L}} + f\,  \Delta\,\tilde{L} + \frac{s}{2}\Delta f
\end{align*} Then invoking the Forward Kolmogorov equation ${\partial_t f} = \mathcal{L}^* f$, one obtains the following Fokker--Planck--Smoluchowski PDE governing the evolution of the density of the SDE, 
\begin{align}\label{def:fps} \pdv{\rho_s}{t} &= \langle \nabla \rho_s, \nabla \tilde{L} \rangle + \rho_s \Delta \tilde{L} + \frac{s}{2} \Delta \rho_s  \quad \quad \text{(FPS)}
\end{align} 
Further, under appropriate conditions on $\tilde{L}$ the above implies that the density $\rho_s(t)$ converges exponentially fast to the Gibbs' measure corresponding to the objective function i.e the distribution with p.d.f \[\mu_s = \frac{1}{Z_s}{\exp\left(-\frac{2 \tilde{L}(\W)}{s}\right)}\] where $Z_s$ is the normalization factor. The sufficient conditions on $\tilde{L}$ that were shown to be needed to achieve this ``mixing" and to know a rate for it, are that of $\tilde{L}$ being a ``Villani Function'' as defined below, 

\begin{definition}[{\bf Villani Function} (\cite{villani2009hypocoercivity,weijie_sde})]\label{def:villani}
A map $f : \R^d \rightarrow \R$ is called a Villani function if it satisfies the following conditions,
\begin{enumerate}
\item  $f \in C^\infty$
\item  $\lim_{\norm{\x}\rightarrow \infty} f(\x) = +\infty$
\item  ${\displaystyle \int_{\R^d}} \exp\left({-\frac{2f(\x)}{s}}\right) \dd{\x} < \infty ~\forall s >0$
\item $\lim_{\norm{\x} \rightarrow \infty} \left ( -\Delta f(\x) + \frac{1}{s} \cdot \norm{\nabla f(\x)}^2 \right ) = +\infty ~\forall s >0$
\end{enumerate} Further, any $f$ that satisfies conditions 1 -- 3 is said to be ``confining''.
\end{definition}

 From Lemma $5.2$ \cite{weijie_sde}, the empirical or the population risk, $\tilde{L}$, being confining is sufficient for the FPS PDE (equation \ref{def:fps}) to evolve the density of SGD--SDE (equation \ref{def:sde}) to the said Gibbs' measure. 

But, to get non-asymptotic guarantees of convergence (Corollary $3.3$, \cite{weijie_sde}) -- even for the SDE, we need a Poincar\'e--type inequality to be satisfied (as defined below) by the aforementioned Gibbs' measure $\mu_s$. A sufficient condition for this Poincar\'e--type inequality to be satisfied is if a confining loss function $\tilde{L}$ also satisfied the last condition in definition \ref{def:villani} (and is consequently a Villani function).

\begin{theorem}[Poincar\'e--type Inequality (\cite{weijie_sde})]\label{def:lambdas}
Given a $f : \R^d \rightarrow \R$ which is a Villani Function (Definition \ref{def:villani}), for any given $s>0$, define a measure with the density, $\mu_s (\x) = \frac{1}{Z_s}{\exp\left(-\frac{2 f(\x)}{s}\right)}$, where $Z_s$ is a normalization factor. Then this (normalized) Gibbs' measure $\mu_s$ satisfies a Poincare-type inequality i.e $\exists ~\lambda_s >0$ (determined by $f$) s.t $\forall h \in C_c^{\infty}(\mathbb{R}^d)$ we have,
\[ {\rm Var}_{\mu_s}[h] \leq \frac{s}{2  \lambda_s} \cdot \E_{\mu_s} [ \norm{\nabla h}^2]\] 
\end{theorem}

The approach of \cite{weijie_sde} has certain key interesting differences from many other contemporary uses of SDEs to prove the convergence of discrete time stochastic algorithms. Instead of focusing on the convergence of parameter iterates $\mathbf{W}^k$, they instead look at the dynamics of the expected error i.e $\E [ \tilde{L}(\mathbf{W}^k)]$, for  $\tilde{L}$ the empirical or the population risk. This leads to a transparent argument for the convergence of $\E [ \tilde{L}(\mathbf{W}^k)]$ to $\inf_{\mathbf{W}}  \tilde{L}(\mathbf{W})$, by leveraging standard results which help one pass from convergence guarantees on the SDE to a convergence of the SGD. 
 
We note that \cite{weijie_sde} achieve this conversion of guarantees from SDE to SGD by additionally assuming gradient smoothness of $\tilde{L}$ -- and we would show that this assumption holds for the natural neural net loss functions that we consider.

\section{Proof of Theorem \ref{lem:error_bound}}\label{sec:proof_mainthm} 
\begin{proof}
\quad Note that $\tilde{L}$ being a confining function can be easily read off from Definition \ref{def:villani}. Property (2) in Definition \ref{def:villani} is straightforward to verify, since  $\tilde{L}_i \geq 0 \,\, \forall \, i$ by definition, and the Frobenius--norm regularizer ensures that the total loss is unbounded above.

To verify Property (3), consider the following (noting that $\tilde{L}_i \geq 0$ by definition)
\begin{align*}
    \int \exp\left(-\frac{2 \tilde{L}(\W)}{s}\right) \, d\W &\leq \int \exp\left(\frac{-2}{s} \cdot \frac{\lambda}{2} \norm{\W}_F^2\right) \, d\W
    = \int \exp\left(\frac{-\lambda}{s} \norm{\W}_F^2\right) \, d\W
\end{align*} 
It is clear from above that the integral is  $O\left(\poly(\frac{s}{\lambda})\right)$.

Further, as shown in Appendix \ref{villani_condition}, the following inequalities hold,
\begin{align}
    \nonumber  \norm{\nabla_{\mathbf{W}}\tilde{L}}_2^2 &\geq (\lambda^2-2\lambda  M_D L B_x^2 \norm{\a}_2^2) \norm{\mathbf{W}}_F^2 \\&- {2\lambda m M_D B_x \norm{\a}_2 (B_y + \norm{\a}_2\norm{\c}_2) \norm{\mathbf{W}}_F}\\
    \Delta_{\mathbf{W}\mathbf{W}}\tilde{L} &\leq p\left[M_d^2 B_x^2 \norm{\a}_2^2 + {\norm{\a}_2}\left[ \left(B_y + \norm{\a}_2\left(\norm{\c}_2 + LB_x\norm{\mathbf{W}}_F\right)\right)\left(M_D'B_x^2\right)\right] + \lambda d \right]
\end{align}

Combining the above two inequalities we can conclude that, $\exists$ functions $g_1, g_2, g_3$ such that,

\begin{align*}
\frac{1}{s} \norm{\nabla_{\mathbf{W}}\tilde{L}}^2 - \Delta_{\mathbf{W}\mathbf{W}
}\tilde{L} &\geq g_1(\lambda, s) \norm{\mathbf{W}}_F^2 - g_2(\lambda, s) \norm{\mathbf{W}}_F + g_3(\lambda, s)
\end{align*} 

where in particular,
\[g_1(\lambda, s) = \lambda^2-2\lambda \cdot  M_D L B_x^2 \norm{\a}_2^2.\]

Hence we can conclude that  for $\lambda > \lambda_{c} \coloneqq 2 M_D L B_x^2 \norm{\a}_2^2, \forall s > 0,$ $\frac{1}{s} \norm{\nabla_{\mathbf{W}}\tilde{L}}^2 - \Delta_{\mathbf{W}\mathbf{W}
}\tilde{L}$ diverges as $\norm{\mathbf{W}} \rightarrow +\infty$, since $g_1(\lambda, s) > 0.$ The key aspect of the above analysis being that the bound on $\Delta_{\mathbf{W}\mathbf{W}}$ does not 
 depend on $\norm{\mathbf{W}}_F^2.$

Thus we have, that the following limit holds,
\begin{align*}
    \lim_{\norm{\mathbf{W}}_F \rightarrow +\infty} \left(\frac{1}{s} \norm{\nabla_{\mathbf{W}}\tilde{L}}^2 - \Delta_{\mathbf{W}\mathbf{W}
}\tilde{L}\right) = +\infty
\end{align*} for the range of $\lambda$ as given in the theorem, hence proving that $\tilde{L}$ is a Villani function. 

Towards getting an estimate of the step-length as given in the theorem statement, we also show in Appendix \ref{smoothness}  that the loss function $\tilde{L}$ is gradient--Lipschitz with the smoothness coefficient being upperbounded as, 
\[{\rm gLip}(\tilde{L}) \leq  \sqrt{p}\left({\norm{\a}_2 B_x} B_y L_{\sigma}' +  \sqrt{p}\norm{\a}_2^2 M_D^2 B_x^2 + {{p} \norm{\a}_2^2 B_x^2} M_D' B_{\sigma} + \lambda\right).\]

Now we can invoke Theorem $3$ (Part 1), \cite{weijie_sde} to write the error of running SGD at a constant step-size $0 < s \leq \frac{1}{{\rm gLip}(\tilde{L})}$ for $\frac{T}{s}$ iterations as, 

\[ {\E}\tilde{L}(\mathbf{W}^\frac{T}{s}) - \min_{\mathbf{W}}\tilde{L}(\mathbf{W}) \leq (A(\tilde{L}) + B(T,\tilde{L})) \cdot s + C(s,\tilde{L}) \norm{\rho - \mu_s}_{\mu_s^{-1}} e^{ - \lambda_s \cdot T} \]

Hence to recover the guarantee in Theorem \ref{lem:error_bound}, it suffices to choose the step-size as, \[s = s^*(\epsilon,T) \coloneqq \min\left(\frac{1}{{\rm gLip}(\tilde{L})}, \frac{\epsilon}{(A(\tilde{L}) + B(T,\tilde{L}))} \right)\]

We recall that in above the time horizon $T$ is a free parameter that determines the step-size $s^*$. For details about the constants $A, B ~\& C,$ the reader is referred to Appendix \ref{sec:constants}.
\end{proof}

\section{Ablation Study with Noisy Labels}\label{sec:experiments}

In this section, we give some further experimental studies of doing regression over nets that are within the ambit of our core result of Theorem \ref{lem:error_bound}. The demonstrations in this section are designed to address two conceptual issues,  

\begin{enumerate}
    \item \,\, Firstly, we verify that when using sigmoid gates and $\lambda$ slightly larger than the theoretically computed threshold of, $\lambda_{c,1}^{si} = 0.125$ (equation \ref{lambsigvil}) it does not lead to the regularizer term in the loss overpowering the actual (empirical risk) objective. We demonstrate this by adding random noise (details follow) to the training labels and observing that the training / test loss values reached by the SGD degrades in response to increasing the fraction of noisy labels even at $\lambda = 0.13 > \lambda_{c,1}^{si}$. 
    \item \, Secondly, we verify that even when the neural net is not initialized from $\rho_{initial}$ as described in Corollary \ref{thm:sgd-sig}, the SGD converges and is affected by $\lambda$ tuning as would be expected from the theorem.
\end{enumerate}
We train the neural net on 3 different synthetic training datasets: a) the clean version (generated as $\x \sim {U}[0, 1)^d, \, y = \sin{\left(\pi \frac{\norm{\x}_2^2}{d}\right)}$ and when b) $50\%$ and c) $90\%$ of the labels in the clean data have been additively corrupted by $0.05\cdot \xi, \,  \xi \sim \text{Cauchy}(0, 1).$ 
 

Recalling the theoretically computed threshold of, $\lambda_{c,1}^{si} = 0.125$ (equation \ref{lambsigvil}) we choose $\lambda  = 0.013, 0.13$ as two values above and below the threshold. We set $\norm{\a}_2 = \frac{1}{B_x}$, $d=20$ (the data dimension). Further, we choose data such that $B_x = \sqrt{d}$, and choose 2 values for the number of gates $p$ -- one which is half the data-dimension ($10$) and one which is more than double of it ($50$). Finally, the values of $p$, step-size $\eta$, and $\lambda$ are specified for each of the plots (in caption) in Figs. \ref{fig:013_10}, \ref{fig:13_10}, \ref{fig:013_50}, \ref{fig:13_50}.




{\it In all experiments we see that the test loss on clean data (the middle figures in the panels below) progressively degrades with increasing the fraction of noisy labels in the training data -- thus confirming that using $\lambda = 0.13 > \lambda_{c,1}^{si}$ regularization did not obfuscate the algorithm's response to meaningful details of the unregularized loss.}

Code for these experiments can be found at this \href{https://github.com/pulkitgopalani/villani_sgd.git}{Github repository}.

\begin{figure}[htb]
    \centering
    \includegraphics[width=\textwidth]{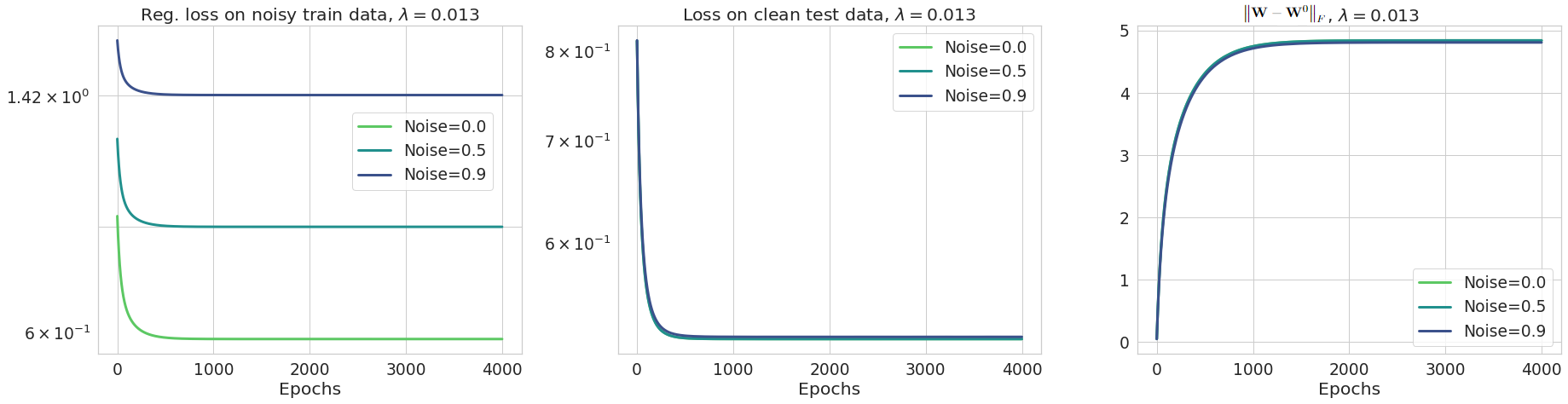}
    \caption{Regression, $(p, \eta, \lambda) = (10, 1\text{e-}2, 0.013)$}
    \label{fig:013_10}
\end{figure}

\begin{figure}[htb]
    \centering
    \includegraphics[width=\textwidth]{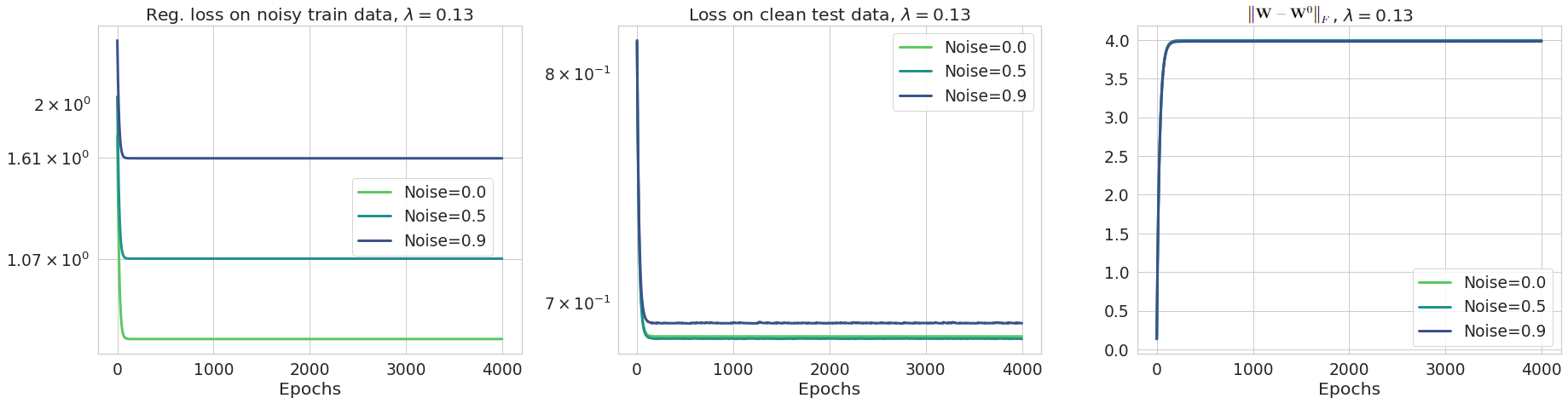}
    \caption{Regression, $(p, \eta, \lambda) = (10, 1\text{e-}2, 0.13)$}
    \label{fig:13_10}
\end{figure}

\begin{figure}[htb]
    \centering
    \includegraphics[width=\textwidth]{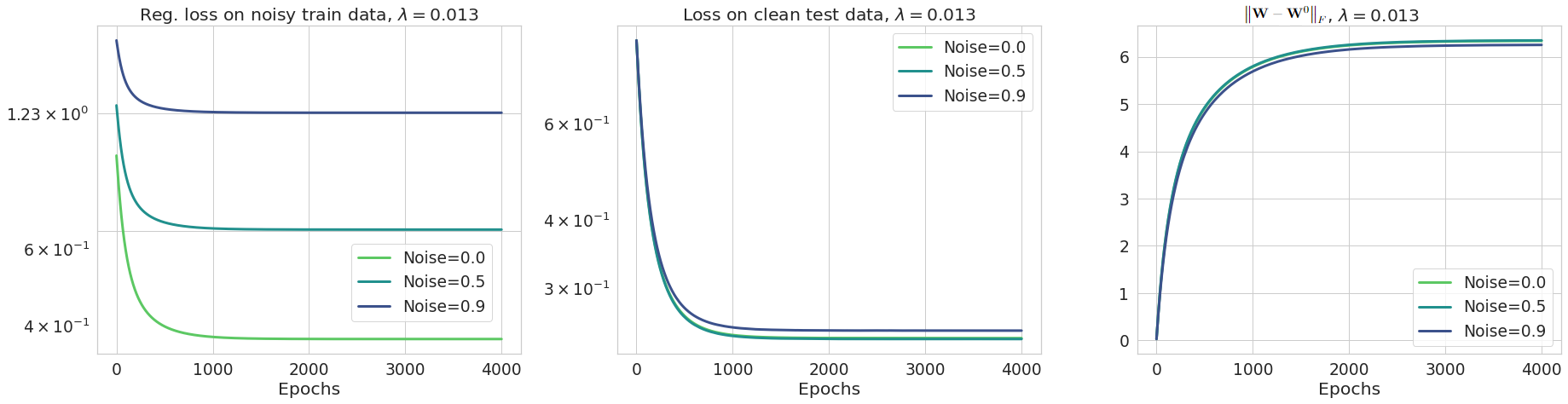}
    \caption{Regression, $(p, \eta, \lambda) = (50, 5\text{e-}3, 0.013)$}
    \label{fig:013_50}
\end{figure}

\begin{figure}[htb]
    \centering
    \includegraphics[width=\textwidth]{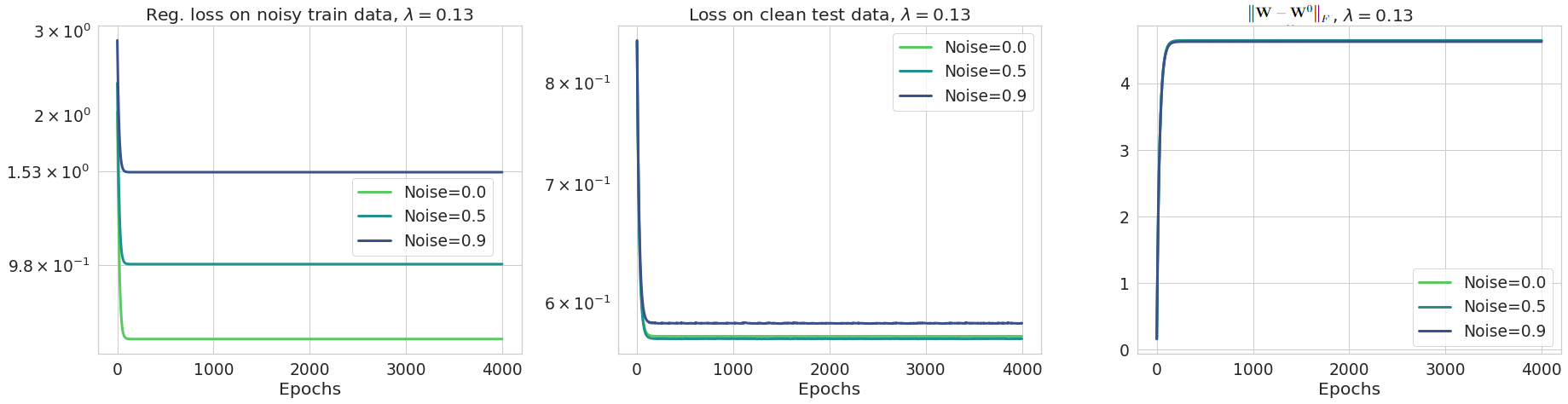}
    \caption{Regression, $(p, \eta, \lambda) = (50, 1\text{e-}2, 0.13)$}
    \label{fig:13_50}
\end{figure}

\section{Conclusion}\label{sec:conc}

Convergence of discrete time algorithms like SGD to their continuous time counterpart (SDEs) has lately been an active field of research. Availability of a well--developed mathematical theory for SDEs holds potential for this mapping to makes theoretical analyses of stochastic gradient--based algorithms possible in hitherto unexplored regimes \textit{if} the discrete time algorithm has a corresponding SDE to which its proximity is quantifiable.

In a recent notable progress in \cite{sanjeev_svag}, an iterative algorithm called SVAG (Stochastic Variance Amplified Gradient) was introduced as a quantifiably good discretization of another SGD motivated SDE than what we use here. They show that SVAG iterates converge in expectation to the covariance--scaled It\^{o} SDE. And they gave empirical evidence that this discretization and their SDE are also tracking SGD on real nets. It remains to be explored in future if this can become a pathway towards better guarantees on neural training than what we get here.  

Additionally, we note that since SoftPlus is not bounded, using our current technique it does not follow that the SGD algorithm also converges to the global minimizer of its Frobenius norm regularized loss. Investigating the possibility of this result could be an exciting direction of future work. In general we believe that trying to reproduce our Theorem \ref{lem:error_bound} using a direct analysis of the dynamics of SGD could be a fruitful venture leading to interesting insights. Lastly, our result motivates a new direction of pursuit in deep-learning theory, centered around understanding the nature of the Poincar\'e constant of Gibbs' measures induced by neural nets.

\subsection*{Acknowledgments}
We would like to thank Hadi Daneshmand and Zhanxing Zhu for their critical suggestions with setting up the experiments. Our work was hugely helped by Hadi sharing with us some of his existing code probing similar phenomenon as explored in our Section \ref{sec:experiments}. We thank Matthew Colbrook and Siva Theja Maguluri for extensive discussions throughout this project. We are also grateful to Weijie Su, Siddhartha Mishra, Avishek Ghosh, Theodore Papamarkou, Alireza and Purushottam Kar for insightful comments at various stages of preparing this draft. During parts of doing this work Anirbit Mukherjee was also helped by the Royal Society Research Grant RGS R1 231332.

\subsection*{Data Availability Statement}
The data used for experiments in this paper can be generated using the accompanying code provided earlier.

\bibliography{References.bib}

\appendix

\renewcommand{\theHsection}{A\arabic{section}}

\section{Towards Establishing the Villani condition for the Empirical Loss on Nets}

\label{villani_condition} In the following, $\norm{\mathbf{W}}$ for a matrix $\mathbf{W}$ denotes its spectral or operator norm. We recall that the  regularized $\ell_2-$loss for training the given neural net on data $\mathcal{D} = \{(\x_i, y_i)\}_{i=1}^n$ is,
\[\tilde{L}(\mathbf{W}) \coloneqq  \frac{1}{n}\sum_{i=1}^n \tilde{L}_i(\mathbf{W}) + R_{\lambda}({\mathbf{W}}) \coloneqq \frac{1}{n}\sum_{i=1}^n \left \{ \frac{1}{2} \left(y_i - f(\x_i; \a,\mathbf{W}) \right)^2 \right \} + \frac{\lambda}{2} \norm{\mathbf{W}}_F^2 \]
Explicitly stated, the SGD iterates with step-length $s >0$ that we analyze on the above loss are, 
\begin{align*}
\mathbf{W}^{k+1} &= \mathbf{W}^k - s \left[\frac{1}{b}\sum_{i \in \mathcal{B}_k} \nabla \tilde{L}_i(\mathbf{W}^k) + \nabla R_{\lambda}({\mathbf{W}^k})\right]\\
&= \mathbf{W}^k - \frac{s}{b}\sum_{i \in \mathcal{B}_k}  \Big ( - \left(y_i - f(\x_i; \a,\mathbf{W}^k) \right) \cdot \nabla_{\mathbf{W}^k}    f(\x_i; \a,\mathbf{W}^k) \Big ) -  s\lambda \mathbf{W}^k \\
&= (1 - s \lambda) \mathbf{W}^k + \frac{s}{b} \sum_{i \in \mathcal{B}_k} \left [ (y_i - f(\x_i; \a,\mathbf{W}^k)) \cdot \nabla_{\mathbf{W}^k}  f(\x_i; \a,\mathbf{W}^k)  \right ]
\end{align*}

We consider gradients and Laplacians of $\tilde{L}(\mathbf{W})$ w.r.t. $\mathbf{W}_j, j=1,2,\ldots,p$  separately, 

\begin{align*}
    \nabla_{\mathbf{W}_j}\tilde{L} &=  \left[\frac{1}{n}\sum_{i=1}^n \left(f(\x_i;\, \a, \mathbf{W}) - y_i\right)\nabla_{\mathbf{W}_j} f(\x_i;\, \a, \mathbf{W}) \right] + \lambda \mathbf{W}_j\\
    &=  \left[\frac{1}{n}\sum_{i=1}^n a_j \left(f(\x_i;\, \a, \mathbf{W}) - y_i\right)  \sigma'(\mathbf{W}_j^\top \x_i) \x_i \right] + \lambda \mathbf{W}_j \quad ; \quad j=1,2,\ldots,p \\
    \implies \norm{\nabla_{\mathbf{W}_j}\tilde{L}}_2^2 &= \norm{\frac{1}{n}\sum_{i=1}^n a_j \left(f(\x_i;\, \a, \mathbf{W}) - y_i\right)  \sigma'(\mathbf{W}_j^\top \x_i) \x_i}_2^2 + \lambda^2 \norm{\mathbf{W}_j}_2^2 + \\&  \quad \quad \quad \quad\quad \quad \quad \quad 2\ip{\left[\frac{1}{n}\sum_{i=1}^n a_j \left(f(\x_i;\, \a, \mathbf{W}) - y_i\right)  \sigma'(\mathbf{W}_j^\top \x_i) \x_i\right]}{\lambda \mathbf{W}_j}\\
    &\geq \lambda^2 \norm{\mathbf{W}_j}_2^2 - \frac{2\lambda}{n}\ip{\left[\sum_{i=1}^n a_j \left(y_i - f(\x_i;\, \a, \mathbf{W})\right)  \sigma'(\mathbf{W}_j^\top \x_i) \x_i\right]}{\mathbf{W}_j} \\
    &\geq \lambda^2 \norm{\mathbf{W}_j}_2^2 - \frac{2\lambda\norm{\mathbf{W}_j}_2}{n}\left[\sum_{i=1}^n  \abs{a_j} \cdot \left(\abs{y_i} + \norm{\a}_2\norm{\sigma(\mathbf{W}\x_i)}_2\right) \cdot \abs{\sigma'(\mathbf{W}_j^\top \x_i) } \cdot  \norm{\x_i}_2\right]
\end{align*} Where in the last inequality we have used Cauchy-Schwarz inequality and triangle inequality wherever applicable. We have,
\begin{align*}
    \norm{\nabla_{\mathbf{W}_j}\tilde{L}}_2^2 &\geq \lambda^2 \norm{\mathbf{W}_j}_2^2 - \norm{\mathbf{W}_j}_2\abs{a_j} \cdot \Big (  2\lambda M_D B_x  \left \{ B_y + \norm{\a}_2 \cdot (LB_x\norm{\mathbf{W}} + \norm{\c}_2) \right \} \Big )
\end{align*}

In above we sum over $j=1,2,\ldots,p$ on both sides. Via Cauchy-Schwartz inequality we upperbound as, $\sum_{j=1}^p \norm{\mathbf{W}_j}_2 \cdot \abs{a_j} \leq \norm{\mathbf{W}}_F \cdot \norm{\a}$ to get,

\begin{align}\label{eg:grad-norm}
    \nonumber \norm{\nabla_{\mathbf{W}}\tilde{L}}^2 &\geq \lambda^2 \norm{\mathbf{W}}_F^2 - {\norm{\a}_2  \norm{\mathbf{W}}_F} \cdot 2\lambda M_D B_x \cdot \left[B_y + (LB_x\norm{\mathbf{W}} + \norm{\c}_2)\norm{\a}_2\right]\\
    \nonumber &\geq \lambda^2 \cdot \norm{\mathbf{W}}_F^2  -2\lambda M_D L B_x^2 \norm{\a}_2^2 \cdot \norm{\mathbf{W}}_F \norm{\mathbf{W}} - {2\lambda  M_D B_x \norm{\a}_2 (B_y + \norm{\c}_2 \cdot \norm{\a}_2 ) \norm{\mathbf{W}}_F}\\
    &\geq \lambda^2 \cdot \norm{\mathbf{W}}_F^2  -2\lambda M_D L B_x^2 \norm{\a}_2^2 \cdot \norm{\mathbf{W}}_F^2 - {2\lambda  M_D B_x \norm{\a}_2 (B_y + \norm{\c}_2 \cdot \norm{\a}_2 ) \norm{\mathbf{W}}_F}
\end{align}

In the last line above we used that, $\norm{\mathbf{W}} \leq \norm{\mathbf{W}}_F$.

Now, for $j=1,2,\ldots,p$ we have for the second derivatives, 
\begin{align*}
    \Delta_{\mathbf{W}_j\mathbf{W}_j}\tilde{L} = \nabla_{\mathbf{W}_j} \cdot \nabla_{\mathbf{W}_j}\tilde{L} &= \nabla_{\mathbf{W}_j} \cdot \left[\left\{\frac{1}{n}\sum_{i=1}^n \left(f(\x_i;\, \a, \mathbf{W}) - y_i)\right)\nabla_{\mathbf{W}_j} f(\x_i;\, \va, \mathbf{W})\right\} + \lambda \mathbf{W}_j\right] \\ 
    &= \left[\frac{1}{n}\sum_{i=1}^n \nabla_{\mathbf{W}_j} \cdot\left\{\left(f(\x_i;\, \va, \mathbf{W}) - y_i)\right)\nabla_{\mathbf{W}_j} f(\x_i;\, \va, \mathbf{W})\right\} + \lambda d\right]\\
    &= \left[\frac{1}{n}\sum_{i=1}^n \left\{\norm{\nabla_{\mathbf{W}_j} f(\x_i;\, \va, \mathbf{W})}_2^2 + \left(f(\x_i;\, \va, \mathbf{W}) - y_i)\right)\Delta_{\mathbf{W}_j} f(\x_i;\, \va, \mathbf{W})\right\} + \lambda d\right]\\
    &= \left[\frac{1}{n}\sum_{i=1}^n \left\{\norm{a_j \sigma'(\mathbf{W}_j^\top \x_i) \x_i}_2^2 + \left(f(\x_i;\, \va, \mathbf{W}) - y_i)\right)\left(a_j\sigma''(\mathbf{W}_j^\top\x_i)\norm{\x_i}_2^2\right)\right\} + \lambda d\right]\\
    &\leq \left[M_d^2 B_x^2 \norm{\va}_2^2 + \frac{a_j}{n}\sum_{i=1}^n \left\{ \left(f(\x_i;\, \va, \mathbf{W}) - y_i)\right)\left(\sigma''(\mathbf{W}_j^\top\x_i)\norm{\x_i}_2^2\right)\right\} + \lambda d\right]\\
    &\leq \left[M_d^2 B_x^2 \norm{\va}_2^2 + {\norm{\va  }_2}\left[ \left(B_y + \norm{\va}_2\left(\norm{\c}_2 + LB_x\norm{\mathbf{W}}\right)\right)\left(M_D'B_x^2\right)\right] + \lambda d\right]\\
    \implies \Delta_{\mathbf{W}\mathbf{W}}\tilde{L} = \sum_j \Delta_{\mathbf{W}_j\mathbf{W}_j}\tilde{L} &\leq p\left[M_d^2 B_x^2 \norm{\va}_2^2 + {\norm{\va}_2}\left[ \left(B_y + \norm{\va}_2\left(\norm{\c}_2 + LB_x\norm{\mathbf{W}}\right)\right)\left(M_D'B_x^2\right)\right] + \lambda d\right]
\end{align*}

Invoking $\norm{\mathbf{W}} \leq \norm{\mathbf{W}}_F$ in the last line above we get the expression as required in the proof in Section \ref{sec:proof_mainthm}.

\section{Bounding the Gradient Lipschitzness Coefficient of the Empirical Risk of the Neural Nets}

\label{smoothness} We start with noting the following equality,
\begin{align*}
    \g_j(\mathbf{W}) \coloneqq \nabla_{\mathbf{W}_j}\tilde{L} &=  \left[\frac{1}{n}\sum_{i=1}^n \left(f(\x_i;\, \va, \mathbf{W}) - y_i\right)\nabla_{\mathbf{W}_j} f(\x_i;\, \va, \mathbf{W}) \right] + \lambda \mathbf{W}_j
\end{align*} We first determine a bound on the Lipschitz constant of $\g_j$ for $j=1,2,\ldots,p,$.

For any two possible weight matrices $\mathbf{W}_1$ and $\mathbf{W}_2$ we have, 

\begin{align*}
    &\norm{\g_j(\mathbf{W}_2) - \g_j(\mathbf{W}_1)}_2\\
    &= \norm{\left[\frac{1}{n}\sum_{i=1}^n \left(f(\x_i;\, \va, \mathbf{W}_2) - y_i\right)\nabla_{\mathbf{W}_{2,j}} f(\x_i;\, \va, \mathbf{W_2}) - \left(f(\x_i;\, \va, \mathbf{W}_1) - y_i\right)\nabla_{\mathbf{W}_{1,j}} f(\x_i;\, \va, \mathbf{W}_1) \right] + \lambda (\mathbf{W}_{2,j} - \mathbf{W}_{1,j})}_2\\
    &\leq \bigg[\frac{1}{n}\sum_{i=1}^n \norm{\left(f(\x_i;\, \va, \mathbf{W}_2) - y_i\right)\nabla_{\mathbf{W}_{2,j}} f(\x_i;\, \va, \mathbf{W_2}) - \left(f(\x_i;\, \va, \mathbf{W}_1) - y_i\right)\nabla_{\mathbf{W}_{1,j}} f(\x_i;\, \va, \mathbf{W}_1)}_2 \bigg] + \lambda\norm{ \mathbf{W}_{2,j} - \mathbf{W}_{1,j}}_2\\
    &\leq \bigg[\frac{\norm{\va}_2 B_x}{n}\sum_{i=1}^n \abs{\left(f(\x_i;\, \va, \mathbf{W}_2) - y_i\right)\sigma'(\mathbf{W}_{2,j}^\top\x_i) - \left(f(\x_i;\, \va, \mathbf{W}_1) - y_i\right)\sigma'(\mathbf{W}_{1,j}^\top\x_i)} \bigg] + \lambda\norm{ \mathbf{W}_{2} - \mathbf{W}_{1}}_F
\end{align*} 

Hence the problem reduces to determining the Lipschitz constant of $F(\mathbf{W}) \coloneqq \left(f(\x_i;\, \va, \mathbf{W}) - y_i\right)\sigma'(\mathbf{W}_{j}^\top\x_i).$ We can split it as $F(\mathbf{W}) = f(\x_i;\, \va, \mathbf{W})\sigma'(\mathbf{W}_{j}^\top\x_i) +  (-y_i \sigma'(\mathbf{W}_{j}^\top\x_i)).$

We first show that $f(\x_i;\, \va, \mathbf{W})\sigma'(\mathbf{W}_{j}^\top\x_i)$ is Lipschitz. Towards that consider the gradients \[ \h_k(\mathbf{W}) \coloneqq \nabla_{\mathbf{W}_k}\left[f(\x_i;\, \va, \mathbf{W})\sigma'(\mathbf{W}_{j}^\top\x_i)\right], \,\, k=1,2,\ldots,p\]
Now,
\begin{align*}
    \norm{\nabla_{\mathbf{W}_k}\left[f(\x_i;\, \va, \mathbf{W})\sigma'(\mathbf{W}_{j}^\top\x_i)\right]}_2 &= \norm{a_k \sigma'(\mathbf{W}_k^\top \x_i)\sigma'(\mathbf{W}_j^\top \x_i)\x_i + \mathbf{1}_{k=j}\,f(\x_i;\, \va, \mathbf{W}) \sigma''(\mathbf{W}_j^\top\x_i)\x_i}_2\\
    &\leq \norm{\va}_2 M_D^2 B_x + \norm{\va}_2 \sqrt{p} B_{\sigma} M_D' B_x = L_{\rm prod}
\end{align*}
Where in the second term the $\sqrt{p}$ factor comes in from using Cauchy-Schwarz inequality.

We concatenate these functions along the indices $k=1,2,\ldots,p$, to get
\[\h(\mathbf{W}) \coloneqq \nabla_{\mathbf{W}}\left[f(\x_i;\, \a, \mathbf{W})\sigma'(\mathbf{W}_{j}^\top\x_i)\right] = \left[\h_1(\mathbf{W}), \h_2(\mathbf{W}),\ldots,\h_p(\mathbf{W}) \right]\]
Thus we have,
\begin{align*}
    \norm{\h(\mathbf{W}_2) - \h(\mathbf{W}_1)}_2 &= \norm{\left[\h_1(\mathbf{W}_2),\h_2(\mathbf{W}_2),\ldots,\h_p(\mathbf{W}_2)\right] - \left[\h_1(\mathbf{W}_1),\h_2(\mathbf{W}_1),\ldots,\h_p(\mathbf{W}_1)\right]}_2\\
    &= \norm{\left[\h_1(\mathbf{W}_2) - \h_1(\mathbf{W}_1),\h_2(\mathbf{W}_2) - \h_2(\mathbf{W}_1),\ldots,\h_p(\mathbf{W}_2) - \h_p(\mathbf{W}_1)\right]}_2\\
    &= \bigg|\bigg| \Big[\norm{\h_1(\mathbf{W}_2) - \h_1(\mathbf{W}_1)}_2,\norm{\h_2(\mathbf{W}_2) - \h_2(\mathbf{W}_1)}_2,\ldots,\norm{\h_p(\mathbf{W}_2) - \h_p(\mathbf{W}_1)}_2\Big] \bigg|\bigg|_2\\
    &\leq \norm{\left[L_{\rm prod}\norm{\mathbf{W}_2-\mathbf{W}_1}_F, L_{\rm prod}\norm{\mathbf{W}_2-\mathbf{W}_1}_F,\ldots,L_{\rm prod}\norm{\mathbf{W}_2-\mathbf{W}_1}_F\right]}_2\\
    &= \sqrt{p}L_{\rm prod}\norm{\mathbf{W}_2-\mathbf{W}_1}_F
\end{align*}
Thus, the Lipschitz constant of $f(\x_i;\, \a, \mathbf{W})\sigma'(\mathbf{W}_{j}^\top\x_i)$ is $\sqrt{p}\left(\norm{\a}_2 M_D^2 B_x + \norm{\a}_2 \sqrt{p}B_{\sigma}  M_D' B_x \right).$

The Lipschitz constant of $y_i \sigma'(\mathbf{W}_j^\top\x_i)$ is simply $\abs{y_i} L_{\sigma}'.$ Combining these using the fact that the Lipschitz constant of $f(x)+g(x)$ for functions $f(x)$ ($L_1-$Lipschitz) and $g(x)$ ($L_2-$Lipschitz) is $L_1 + L_2,$ we have that a common upperbound on the Lipschitz constant of $\g_j(\mathbf{W}), \,\, j=1,2,\ldots,p$ (say $L_{\rm row}$) can be given as, 
\[L_{\rm row} = {\norm{\a}_2 B_x} B_y L_{\sigma}' +  \sqrt{p}\norm{\a}_2^2 M_D^2 B_x^2 + {{p} \norm{\a}_2^2 B_x^2} M_D' B_{\sigma} + \lambda\]

Proceeding as in the case of $\h(\mathbf{W})$ above, we now  concatenate the above gradients $\g_j$ w.r.t the index $j$ in a vector form (of dimension $pd$) to get the following $pd-$dimensional gradient vector of the empirical loss, 
\[\nabla_{\mathbf{W}}\tilde{L} = \g(\mathbf{W}) \coloneqq \left[\g_1(\mathbf{W}),\g_2(\mathbf{W}),\ldots,\g_p(\mathbf{W})\right]\] and the Lipschitz constant of $\g$ - and hence the gradient Lipschitz constant for $\tilde{L}$ to be bounded as, 
\[{\rm gLip}(\tilde{L}) \leq \sqrt{p}\left({\norm{\a}_2 B_x} B_y L_{\sigma}' +  \sqrt{p}\norm{\a}_2^2 M_D^2 B_x^2 + {{p} \norm{\a}_2^2 B_x^2} M_D' B_{\sigma} + \lambda\right)\]

Thus we get the expression as required in the proof in Section \ref{sec:proof_mainthm}.

\section{Constants in the convergence guarantee}\label{sec:constants}

The constants $A, B$ are similar to those defined in \cite{weijie_sde}; following \cite{weijie_sde}, the constant $C$ in the guarantee in Theorem \ref{lem:error_bound} is defined as follows
\[C(s,\tilde{L}) = \left(\int_{\mathbb{R}^{p\times d}} (\tilde{L}(\mathbf{W}) - \min_{\mathbf{W}}\tilde{L}(\mathbf{W}))^2 \mu_s (\W) \, d\mathbf{W}\right)^{1/2}\]

for $\mu_s(\mathbf{W})$ being the Gibbs' measure, $\mu_s(\mathbf{W}) = \frac{1}{Z_{s}}\exp\left({-\frac{2\tilde{L}(\mathbf{W})}{s}}\right)$ with $Z_{s}$ being the normalization factor. 

For determining $\lambda_s$, \cite{weijie_sde} consider the function $V_s(\mathbf{W}) = \norm{\nabla \tilde{L}}^2 / s - \Delta \tilde{L}.$ Let $R_{0,s} > 0$ be large enough such that $V_s(\mathbf{W}) > 0$ for $\norm{\mathbf{W}}_F \geq R_{0,s}.$ For $R_s > R_{0,s},$ \cite{weijie_sde} define $\epsilon(R_s)$ as \[\epsilon(R_s) = \frac{1}{\inf\{V_s(\mathbf{W}) \, : \, \norm{\mathbf{W}}_F \geq R_s\}}\] where $R_s$ is assumed large enough such that $\int_{\norm{\mathbf{W}}_F \leq R_s} d\mu_s \geq 1/2.$ 
For $B_{R_s}$ as the ball of radius ${R_s}$ centered at origin in $\mathbb{R}^{p \times d}$, \cite{weijie_sde} define \[\mu_{s,R_s} = \left[\int_{\norm{\mathbf{W}}_F\leq R_s} d\mu_s(\mathbf{W}) \right]^{-1} \mu_s(\mathbf{W}) \mathbf{1}_{\norm{\mathbf{W}}_F\leq R_s}.\] Using the Poincar\'e inequality in a bounded domain[\cite{evans_pde}, Theorem 1, Chapter 5.8], \cite{weijie_sde} define the constant $C(R_s)$ to be s.t the the following holds $\forall h \in C_c^{\infty}(\mathbb{R}^d)$, 
\[\int_{\mathbf{W} \in \mathbb{R}^{p \times d}} h^2 d\mu_{s,R_s} \leq s \cdot C(R_s) \int_{\mathbf{W} \in  \mathbb{R}^{p \times d}} \norm{\nabla h}^2 \mu_{s,R_s} d\mu_{s,R_s} + \left(\int_{\mathbf{W} \in  \mathbb{R}^{p \times d}} h \,  d\mu_{s,R_s}\right)^2 \]

Then the key quantity $\lambda_s$ occurring in the aforementioned convergence guarantee for SGD was shown to be, \[\lambda_s = \frac{1 + 3s \left(\inf_{\mathbf{W} \in  \mathbb{R}^{p \times d}} V_s(\mathbf{W})\right) \epsilon(R_s)}{2(C(R_s) + 3\epsilon(R_s))}\]

\clearpage  
\appendix

\end{document}